%% file: main.tex
\documentclass[11pt]{article}

\usepackage[preprint]{acl}
\usepackage[T1]{fontenc}

\usepackage[utf8]{inputenc}

\usepackage{microtype}

\usepackage{inconsolata}

\usepackage{graphicx}

\usepackage{amsmath}
\usepackage{amssymb}
\usepackage{amsthm}
\usepackage{enumitem}
\usepackage{algpseudocode}
\usepackage{booktabs}       
\usepackage{subcaption}
\usepackage{nameref}      
\usepackage{algorithm}
\usepackage{multirow}       
\usepackage{url}            
\usepackage{hyperref}       
\usepackage{cleveref}
\usepackage{adjustbox}

\newtheorem{theorem}{Theorem}
\Crefname{equation}{Eq.}{Eqs.}

\newcommand{\ours}{REINFORCE++\,}
\newcommand{\varn}{REINFORCE++\textsubscript{\textit{w/} Baseline}\,}

\title{REINFORCE++: Stabilizing Critic-Free\\Policy Optimization with Global Normalization}

\author{Jian Hu \\
\texttt{janhu9527@gmail.com}
\And
Jason Klein Liu \\
\texttt{jasonkleinlove@gmail.com}
\AND
Haotian Xu \\
\texttt{1034351332@qq.com}
\And
Wei Shen \thanks{Corresponding author} \\
\texttt{shenwei0917@126.com }}

\begin{document}
\maketitle
\begin{abstract}
Reinforcement Learning from Human Feedback~(RLHF) plays a crucial role in aligning Large Language Models~(LLMs). The dominant algorithm, Proximal Policy Optimization~(PPO), employs a critic network to estimate advantages, which introduces significant computational and memory overhead. To address this, a family of critic-free algorithms (e.g., GRPO, RLOO) has emerged. However, these methods typically rely on \textit{prompt-level (local)} advantage normalization, which suffers from inaccurate advantage estimation, a tendency to overfit, and, as we show, is a theoretically biased estimator. To solve these challenges, we introduce \textbf{\ours}, a critic-free framework centered on \textbf{Global Advantage Normalization}. By normalizing advantages across the entire global batch rather than small, prompt-specific groups, our method provides a more stable and theoretically sound, \textit{effectively unbiased} estimate (whose bias vanishes as batch size increases). We introduce two variants: \textbf{\ours}, a highly efficient and general algorithm ($k \ge 1$) for general-domain RLHF, and \textbf{\varn}, a robust group-sampling variant ($k > 1$) for complex reasoning tasks. Our empirical evaluation demonstrates that each variant shows superior stability and performance in its respective domain, outperforming existing methods and even PPO in complex agentic settings.
\end{abstract}

\input{section/intro}

\input{section/background}

\input{section/method}

\input{section/experiment}

\input{section/practise}

\input{section/conclusion}

\clearpage

\bibliography{reference}


\appendix

\onecolumn
\begin{center}
    {\Large \bfseries Appendix}
\end{center}

\section{Proof: The GRPO Advantage Estimator is Biased} \label{appendix:proof_GRPO}

\subsection{Assumptions and Settings}
We observe \(N\) rewards \(r_i\) for a prompt, and assume the true baseline is \(\theta\), such that
\[
r_i = \theta + \epsilon_i, \quad \epsilon_i \sim \mathcal{N}(0, \sigma^2),\quad i=1,\dots,N,
\]
with all advantage value \(\epsilon_i\) independent. Define
\[
\bar\epsilon = \frac{1}{N} \sum_{j=1}^N \epsilon_j, \quad
D = \sqrt{ \frac{1}{N} \sum_{j=1}^N (\epsilon_j - \bar\epsilon)^2 }, \quad
A_i = \frac{ \epsilon_i - \bar\epsilon }{D }.
\]
We will prove the following:

\begin{theorem}
For any finite \(N \ge 2\), the advantage estimator \(A_i\) is biased:
\[
\mathbb{E}[A_i \mid \epsilon_i] \ne \epsilon_i.
\]
\end{theorem}

\begin{proof}
We can prove our findings with the following steps:\\
\textbf{Step 1: Numerator Bias}
We rewrite the numerator:
\[
\epsilon_i - \bar\epsilon = \left(1 - \frac{1}{N} \right) \epsilon_i - \frac{1}{N} \sum_{j \ne i} \epsilon_j.
\]
Since the \(\epsilon_j\) for \(j \ne i\) are zero-mean and independent of \(\epsilon_i\),
\[
\mathbb{E}[\epsilon_i - \bar\epsilon \mid \epsilon_i] = \left(1 - \frac{1}{N} \right) \epsilon_i.
\]

\textbf{Step 2: Denominator Depends on \(\epsilon_i\)}

\textit{(a) Compute \(\mathbb{E}[D^2 \mid \epsilon_i]\):}
By definition,
\[
D^2 = \frac{1}{N} \sum_{j=1}^N (\epsilon_j - \bar\epsilon)^2 = \frac{1}{N} \sum_{j=1}^N \epsilon_j^2 - \bar\epsilon^2.
\]
Since
\[
\bar\epsilon = \frac{1}{N} \left( \epsilon_i + \sum_{j \ne i} \epsilon_j \right),
\]
and conditioning on \(\epsilon_i\) keeps the \(\epsilon_j\), \(j \ne i\), i.i.d.\ \(\mathcal{N}(0, \sigma^2)\), we obtain:
\[
\mathbb{E} \left[ \sum_{j=1}^N \epsilon_j^2 \mid \epsilon_i \right] = \epsilon_i^2 + (N-1) \sigma^2,
\]
\begin{align*}
\mathbb{E}[\bar\epsilon^2 \mid \epsilon_i]
&= \frac{1}{N^2} \mathbb{E} \left[ \left( \epsilon_i + \sum_{j \ne i} \epsilon_j \right)^2 \Big| \epsilon_i \right] \\
&= \frac{1}{N^2} \left( \epsilon_i^2 + 2 \epsilon_i \cdot \underbrace{\mathbb{E}\left[ \sum_{j \ne i} \epsilon_j \right]}_{0} + \mathbb{E} \left[ \left( \sum_{j \ne i} \epsilon_j \right)^2 \right] \right) \\
&= \frac{\epsilon_i^2 + (N-1) \sigma^2}{N^2}.
\end{align*}
Subtracting:
\begin{align}
\mathbb{E}[D^2 \mid \epsilon_i]
&= \frac{1}{N} (\epsilon_i^2 + (N-1) \sigma^2) - \frac{\epsilon_i^2 + (N-1)\sigma^2}{N^2} \nonumber \\
&= \underbrace{\frac{(N-1)^2}{N^2} \sigma^2}_{\alpha} + \underbrace{\frac{N-1}{N^2}}_{\beta} \epsilon_i^2 = \alpha + \beta \epsilon_i^2.
\label{eq:condD2}
\end{align}

\textit{(b) \(g(\epsilon_i)\) is Not Constant:}
Let \(g(\epsilon_i) = \mathbb{E}\left[1/D \mid \epsilon_i \right]\) and \(\mu(\epsilon_i) = \mathbb{E}[D^2 \mid \epsilon_i] = \alpha + \beta \epsilon_i^2\).

Using Taylor expansion of \(f(x) = 1/\sqrt{x}\) around \(x_0 = \mu(\epsilon_i)\):
\[
f(x) = \frac{1}{\sqrt{x_0}} - \frac{1}{2} \frac{x-x_0}{x_0^{3/2}} + \frac{3}{8} \frac{(x-x_0)^2}{x_0^{5/2}} + O((x-x_0)^3)
\]

Taking conditional expectation:
\[
\begin{aligned}
g(\epsilon_i) 
&= \mathbb{E}[1/D \mid \epsilon_i] 
= \mathbb{E}\left[\frac{1}{\sqrt{D^2}} \mid \epsilon_i\right] 
= \mathbb{E}\left[f(D^2) \mid \epsilon_i\right], \quad \text{where } f(x) = \frac{1}{\sqrt{x}} \\
&\approx f(\mu(\epsilon_i)) + f'(\mu(\epsilon_i)) \cdot \mathbb{E}[D^2 - \mu(\epsilon_i) \mid \epsilon_i] + \frac{f''(\mu(\epsilon_i))}{2} \cdot \mathbb{E}[(D^2 - \mu(\epsilon_i))^2 \mid \epsilon_i] \\
&= \frac{1}{\sqrt{\mu(\epsilon_i)}} 
- \frac{1}{2 \mu(\epsilon_i)^{3/2}} \cdot \underbrace{\mathbb{E}[D^2 - \mu(\epsilon_i) \mid \epsilon_i]}_{=\,0} 
+ \frac{3}{8 \mu(\epsilon_i)^{5/2}} \cdot \text{Var}(D^2 \mid \epsilon_i) \\
&= \frac{1}{\sqrt{\mu(\epsilon_i)}} + \frac{3}{8} \cdot \frac{\text{Var}(D^2 \mid \epsilon_i)}{\mu(\epsilon_i)^{5/2}}
\end{aligned}
\]

Since \(\mu(\epsilon_i) = \alpha + \beta \epsilon_i^2\) with \(\beta > 0\), the first term alone shows that \(g(\epsilon_i)\) depends on \(\epsilon_i^2\) and hence is not constant.

\textbf{Step 3: Putting It Together}
Decomposing \(A_i\),
\[
A_i = \frac{\epsilon_i - \bar\epsilon}{D} = \left(1 - \frac{1}{N} \right) \frac{\epsilon_i}{D} - \left( \frac{1}{N} \sum_{j \ne i} \epsilon_j \right) \cdot \frac{1}{D}.
\]
For fixed \(\epsilon_i\), the conditional distribution of \(\sum_{j \ne i} \epsilon_j\) is symmetric about zero, while \(1/D\) is always positive. Thus:
\[
\mathbb{E} \left[ \left( \frac{-1}{N} \sum_{j \ne i} \epsilon_j \right) \cdot \frac{1}{D} \Big| \epsilon_i \right] = 0.
\]
It follows that
\[
\mathbb{E}[A_i \mid \epsilon_i] = \left(1 - \frac{1}{N} \right) \epsilon_i \cdot \mathbb{E} \left[ \frac{1}{D} \mid \epsilon_i \right] = \left(1 - \frac{1}{N} \right) \epsilon_i \cdot g(\epsilon_i).
\]

\textbf{Step 4: Concluding the Bias}
If \(A_i\) were unbiased, we would have:
\[
\left(1 - \frac{1}{N} \right) g(\epsilon_i) \equiv 1 \quad \Rightarrow \quad g(\epsilon_i) \equiv \frac{N}{N-1},
\]
which contradicts Step 2. Therefore, for any finite \(N \ge 2\),
\[
\mathbb{E}[A_i \mid \epsilon_i] \ne \epsilon_i.
\]
Hence \(A_i\) is a biased estimator.

\end{proof}

\subsection{Why Use Global Batch Normalization?} \label{appendix:whyglobal}

We observe that as \( N \to \infty \), the denominator of the estimator converges to the constant \( \sigma \), and the bias in the numerator vanishes. This insight underlies our decision to adopt global batch normalization in both \textsc{REINFORCE++} and \textsc{REINFORCE++-Baseline}, as the global batch size ($N_{global}$) is typically much larger (e.g., 1024) than the group batch size ($N_{group}$, e.g., 4 or 8) in practice. As $N_{global} \to \infty$, the sample statistics $\mu_b$ and $\sigma_b$ converge to true constants, making the estimator effectively unbiased.

\clearpage
\section{Algorithm Details}

\subsection{KL Penalty Design}  \label{appendix:kl_design}

In non-critic algorithms like REINFORCE++-Baseline, a KL divergence term is often added as a separate loss to constrain the policy $\pi_{\theta}$ from deviating too far from a reference policy $\pi_{ref}$. In practice, this expectation is estimated using samples, and three common estimators are used.

Given the importance ratio $\delta(y) = \pi_{ref}(y|\cdot) / \pi_{\theta}(y|\cdot)$ (where $\pi_{\theta}$ is the sampling policy and $\pi_{ref}$ is the reference):
\begin{itemize}
    \item $k_{1}(y) = -\log \delta(y) = \log \frac{\pi_{\theta}(y|\cdot)}{\pi_{ref}(y|\cdot)}$
    \item $k_{2}(y) = \frac{1}{2}(\log \delta(y))^{2} = \frac{1}{2}\left(\log \frac{\pi_{ref}(y|\cdot)}{\pi_{\theta}(y|\cdot)}\right)^{2}$
    \item $k_{3}(y) = \delta(y) - 1 - \log \delta(y)$
\end{itemize}

We analyze which of these (k1, k2, or k3) is the correct choice when used as a separate loss term.

\paragraph{Theoretical Reverse KL Gradient}
We are training with samples from $\pi_{\theta}$, so we are approximating the **Reverse KL (RKL)**, $D_{KL}(\pi_{\theta}||\pi_{ref})$. The practical policy gradient for RKL is (more details are in \citep{liu2025rethinking}):
$$
\nabla_{\theta}\mathcal{J}_{RKL}(\theta) = \mathbb{E}_{y\sim\pi_{\theta}(\cdot|x)}\left[\left(\log\frac{\pi_{\theta}(y|x)}{\pi_{ref}(y|x)}\right)\nabla_{\theta}\log \pi_{\theta}(y|x)\right]
$$

\paragraph{k1 Analysis}
We can **exclude k1** as a loss term. Its gradient does not depend on $\pi_{ref}$ and thus provides no constraining effect. The loss $\mathcal{J}_{k{1}} = \mathbb{E}_{y\sim\pi_{\theta}}[\log \pi_{\theta} - \log \pi_{ref}]$ differentiates to $\nabla_{\theta}\mathcal{J}_{k{1}} = \mathbb{E}_{y\sim\pi_{\theta}}[\nabla_{\theta}\log \pi_{\theta}]$, as the $\pi_{ref}$ term vanishes. (Note: k1 *is* used inside the reward for REINFORCE++, but not as a separate loss term).

\paragraph{k2 Analysis}
The k2 estimator provides a gradient that is **equivalent to the theoretical Reverse KL gradient**.

The k2 loss is (noting that $(\log \frac{\pi_{\theta}}{\pi_{ref}})^2 = (\log \frac{\pi_{ref}}{\pi_{\theta}})^2$):
$$
\mathcal{J}_{k{2}as~loss}(\theta)=\mathbb{E}_{y\sim\pi_{\theta}(\cdot|x)}\left[\frac{1}{2}\left(\log\frac{\pi_{\theta}(y|x)}{\pi_{ref}(y|x)}\right)^{2}\right]
$$
Differentiating the term inside the expectation gives:
\begin{align*}
\nabla_{\theta}\left[\frac{1}{2}\left(\log\frac{\pi_{\theta}}{\pi_{ref}}\right)^{2}\right] &= \left(\log\frac{\pi_{\theta}}{\pi_{ref}}\right) \cdot \nabla_{\theta}\left(\log\frac{\pi_{\theta}}{\pi_{ref}}\right) \\
&= \left(\log\frac{\pi_{\theta}}{\pi_{ref}}\right) \cdot \nabla_{\theta}(\log \pi_{\theta} - \log \pi_{ref}) \\
&= \left(\log\frac{\pi_{\theta}}{\pi_{ref}}\right)\nabla_{\theta}\log \pi_{\theta}
\end{align*}
When we take the expectation over $y \sim \pi_{\theta}$, the gradient of the k2 loss becomes:
$$
\mathbb{E}_{y\sim\pi_{\theta}}\left[\left(\log\frac{\pi_{\theta}(y|x)}{\pi_{ref}(y|x)}\right)\nabla_{\theta}\log \pi_{\theta}(y|x)\right]
$$
This **perfectly matches** the practical RKL gradient. Therefore, k2 is the theoretically correct and stable choice.

\paragraph{k3 Analysis}
The k3 gradient is problematic. It is theoretically an estimator for the **Forward KL** ($KL(\pi_{ref}||\pi_{\theta})$), not the Reverse KL. This is a mismatch, as we are sampling from $\pi_{\theta}$, but the Forward KL gradient requires sampling from $\pi_{ref}$.

This method, used in GRPO, suffers from two major issues:
\begin{itemize}
    \item \textbf{Extremely High Variance:} When $\pi_{\theta}(y)$ becomes tiny for a sample $y$ where $\pi_{ref}(y)$ is moderate, the importance weight $\frac{\pi_{ref}(y)}{\pi_{\theta}(y)}$ in the gradient estimator explodes, leading to "infinite variance".
    \item \textbf{Numerical Instability:} It requires computing $p/q$ via $\exp(\log p - \log q)$, which is prone to overflow.
\end{itemize}
This explains why methods using k3 require frequent resetting of $\pi_{ref}$ to prevent the policies from diverging and the k3 estimator from becoming unstable.

\paragraph{Conclusion}
Based on this analysis, the **k2 estimator is the optimal choice** for a KL loss term as it correctly and stably estimates the Reverse KL gradient. We therefore use the $k_2$ estimator in our REINFORCE++-Baseline algorithm.

\subsection{Implementation Tricks} \label{appendix:tricks}
\paragraph{Token-level Advantage}
The advantage $A_{q,o_t}^{norm}$ is computed at the token level. For $t < T$ (where $T$ is the sequence length), the advantage is set to 0. For the final token $t=T$, the advantage is the normalized reward $A_{q,o_T}^{norm}$. This is standard practice in RLHF for LLMs.

\paragraph{Batch Construction}
For REINFORCE++ ($k=1$), a global batch of size $N$ consists of $N$ different prompts. For REINFORCE++-Baseline ($k=4$), a global batch of size $N=1024$ might consist of $1024/4 = 256$ unique prompts. The global mean/std are computed over all $N$ samples.

\paragraph{Mini-Batch Updates}
To enhance training efficiency, we implement mini-batch updates with the following characteristics:
\begin{itemize}[leftmargin=*]
    \item \textbf{Batch Processing:} Data is processed in smaller, manageable chunks rather than full-batch updates.
    \item \textbf{Multiple Updates:} Each mini-batch allows for multiple parameter updates, improving convergence rates.
    \item \textbf{Stochastic Optimization:} Introduces beneficial randomness for better generalization.
\end{itemize}

\paragraph{Reward Normalization and Clipping}

We implement comprehensive reward processing to stabilize training:
\begin{itemize}[leftmargin=*]
    \item \textbf{Normalization:} Standardizes rewards using z-score normalization (our global normalization) to mitigate outliers.
    \item \textbf{Clipping:} Constrains reward values within predefined bounds to avoid instability.
    \item \textbf{Scaling:} Applies appropriate scaling factors for numerical stability during updates.
\end{itemize}

\clearpage

\section{Acknowledgements}
\begin{itemize}[leftmargin=*]
    \item Jian Hu: Conceived the ideas, implemented the REINFORCE++ and \varn algorithms, and contributed to the theoretical proof of the GRPO advantage estimator.  
    \item Jason Klein Liu: Implemented the experimental framework, fine-tuned hyperparameters, wrote the manuscript, and provided GPU resources.  
    \item Haotian Xu: Conducted comparative experiments between \varn and complex tool-calling as well as agent-based scenarios.  
    \item Wei Shen: Supervised the overall project, designed the main experiments, and led the paper writing.  
\end{itemize}

\end{document}

%% file: section/intro.tex
\section{Introduction}

Reinforcement Learning from Human Feedback (RLHF) is a key technique for aligning Large Language Models~(LLMs) with human values and preferences ~\citep{vemprala2023chatgpt,achiam2023gpt, Ouyang2022, shen2024policy,shen2025exploring,hu2024openrlhf}. Despite the emergence of non-RL alternatives like DPO ~\citep{rafailov2023direct}, state-of-the-art applications such as ChatGPT/GPT-4 ~\citep{vemprala2023chatgpt, openai2023gpt4}, Claude ~\citep{anthropic2023introducing}, and Gemini ~\citep{team2023gemini} continue to rely on RL algorithms.

The dominant RL algorithm in this space is Proximal Policy Optimization (PPO) ~\citep{schulman2017proximal}. PPO employs an "Actor-Critic" architecture, where a dedicated critic network is trained to estimate the advantage function (i.e., the extent to which an action yields a higher expected return than the mean). However, this critic network introduces substantial computational overhead and memory demands, making training very expensive and limiting large model alignment in small-scale clusters \citep{shao2024deepseekmath}.

To address PPO's efficiency issues, a family of REINFORCE-based methods has emerged without the critic model, including ReMax~\citep{li2023remax}, RLOO~\citep{2024Back}, and GRPO~\citep{shao2024deepseekmath}. These algorithms remove the critic network and instead estimate the advantage using statistics from multiple responses to the same prompt.

However, this critic-free approach introduces a new, critical challenge: \textit{advantage estimation}. Methods like GRPO and RLOO use \textbf{prompt-level (local) normalization}, calculating the baseline and standard deviation only from the small group of responses generated for a single prompt. This approach suffers from several critical flaws. First, as demonstrated in Appendix \Cref{appendix:proof_GRPO}, the prompt-level advantage estimator is mathematically biased, as the numerator (centered reward) and the denominator (local standard deviation) are not independent. Second, the advantage estimate is sensitive to the small number of samples in the local group (e.g., $k=4$ or $k=8$), where the high variance leads to instability.. If all samples receive similar rewards, the local $\text{std}(\cdot)$ approaches zero, causing the advantage to explode and destabilizing training. Finally, the method encourages overfitting to specific prompts, as the policy is optimized to "win" within its local group rather than achieving a globally high reward, which harms generalization.

We introduce two distinct variants of this framework, each tailored to a specific use case. \ours is an efficient algorithm with $k \ge 1$. In its $k=1$ configuration (detailed in Algorithm 1), it maximizes prompt diversity for general-domain RLHF. It can also be applied with $k>1$ (Section 4.2), using the same global normalization principle. \varn is a specialized \textit{variant} for complex tasks ($k>1$) that benefits from \textbf{group sampling}. It fixes the flaws of GRPO by first subtracting the \textit{group mean} (for reward reshaping) and then normalizing by the \textit{global standard deviation} (for stability). It also employs a more stable $k_2$ for the $\mathrm{KL}$ estimator, as detailed in Appendix \Cref{appendix:kl_design}. In summary, our contributions are as follows:
\begin{itemize}
    \item We provide a theoretical proof that the prompt-level (local) advantage normalization used in methods like GRPO is a biased estimator (Appendix \Cref{appendix:proof_GRPO}).
    \item We propose \ours, a critic-free framework centered on Global Advantage Normalization, as a stable, efficient, and theoretically sound alternative.
    \item We present two specialized algorithms: \ours~($k \ge 1$) for efficient general-purpose RLHF and reasoning, and \varn~($k>1$) for robust complex agentic tasks.
    \item We empirically demonstrate that \ours achieves superior token-efficiency and Out-Of-Distribution~(OOD) generalization in its domain. At the same time, \varn prevents overfitting and outperforms both GRPO and PPO in complex agentic reasoning tasks.
\end{itemize}

%% file: section/background.tex
\section{Background and Related Work}\label{sec:background}

\begin{figure*}[!ht]
    \centering
    \includegraphics[width=0.85\textwidth]{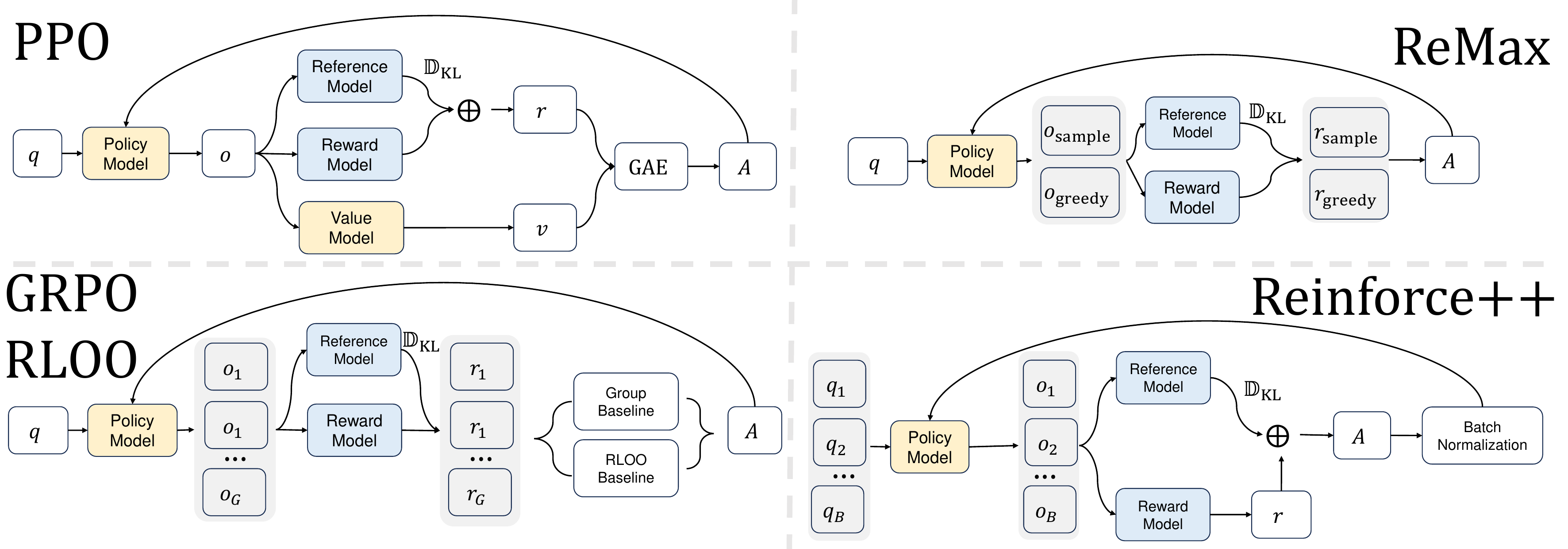}
    \caption{The comparison of PPO, ReMax, GRPO, RLOO, and \ours. \ours, as shown in the figure, removes the critic model and uses global batch normalization. \varn uses group sampling, similar to GRPO/RLOO, but replaces the local group baseline with a group-mean subtraction followed by global batch normalization.}
    \label{fig:main}
\end{figure*}

\subsection{PPO and Critic-Free RLHF}
PPO optimizes LLMs by maximizing the following surrogate objective:

\begin{equation}\label{eq:ppo}
\begin{adjustbox}{max width=\linewidth}
$
\begin{aligned}
\mathcal{L}_{\text{PPO}}(\theta)
&= \mathbb{E}_{q \sim P(Q),\, o \sim \pi_{\theta_{\text{old}}}(O|q)} \Bigg[
    \frac{1}{|o|} \sum_{t=1}^{|o|}
    \min \Big(
        s_t(\theta) A_t, \\[-2pt]
&\quad \text{clip}(s_t(\theta),\, 1 - \epsilon,\, 1 + \epsilon) A_t
    \Big)
\Bigg]
\end{aligned}
$
\end{adjustbox}
\end{equation}

where $s_t(\theta) = \frac{\pi_{\theta}(o_t | q, o_{<t})}{\pi_{\theta_{\text{old}}}(o_t | q, o_{<t})}$ is the probability ratio. The advantage $A_t$ is typically calculated using Generalized Advantage Estimation~(GAE) \citep{schulman2018highdimensionalcontinuouscontrolusing}:
\begin{equation}
    A_{q, o_t} = \sum_{l=0}^{\infty} (\gamma \lambda)^l \delta_{t+l}
\end{equation}
where $\delta_{q, o_t} = r_t + \gamma V(o_{t+1}) - V(o_t)$ is the temporal difference error, and $V(\cdot)$ is the critic network.

Critic-free methods in \Cref{fig:main} remove the critic $V(\cdot)$ and compute the advantage $A_t$ directly from rewards.
\begin{itemize}[leftmargin=*]
    \item \textbf{ReMax} \citep{li2023remax} uses a greedy decoding response $\hat{o}$ as the baseline: $A_{q,o_t} = r(o) - r(\hat{o})$.
    \item \textbf{RLOO} \citep{2024Back} samples $k$ responses and uses the mean of others as the baseline: $A_{q,o_t^{(i)}} = r(o^{(i)}) - \frac{1}{k-1}\sum_{j \neq i} r(o^{(j)})$.
    \item \textbf{GRPO} \citep{shao2024deepseekmath} also samples $k$ responses but normalizes the advantage using the mean and standard deviation of the \textit{local group}:
\end{itemize}
\begin{equation}\label{eq:grpo}
A_{q,o_t^{(i)}} = \frac{r(o^{(i)}) - \text{mean}\bigl(\{r(o^{(j)})\}_{j=1}^{k}\bigr)}{\text{std}\bigl(\{r(o^{(j)})\}_{j=1}^{k}\bigr) + \epsilon}
\end{equation}

\subsection{The Problem with Local Normalization}
The core issue with methods like GRPO lies in \Cref{eq:grpo}. The usage of \textbf{prompt-level (local) normalization} is problematic for three reasons:
\begin{enumerate}
    \item \textbf{Theoretical Bias}: As formally proven in Appendix \Cref{appendix:proof_GRPO}, the estimator is biased. Therefore, the numerator (centered reward) and the denominator (local standard deviation $\text{std}(\cdot)$) are not independent. The denominator's value is correlated with the rewards in the small group, introducing a systematic error in the advantage estimate \citep{mai2025agentrlscalinglaw}.
    \item \textbf{Practical Instability}: In practice, the group size $k$ is small (e.g., 4 or 8). If all sampled responses for a prompt happen to get similar rewards, the local $\text{std}(\cdot)$ approaches zero, causing the advantage to explode. It results in high variance and unstable training.
    \item \textbf{Task Overfitting}: The policy is rewarded for being "better than other samples from the same prompt," not for being "globally good." It can lead to overfitting on simple prompts where it's easy to generate diverse rewards, while failing to improve on complex prompts.
\end{enumerate}

%% file: section/method.tex
\section{Method}\label{sec:method}

Our method addresses the instability of critic-free RLHF by replacing biased local normalization with stable global normalization. We propose two variants for different use cases.

\subsection{Global Normalization}

The primary \ours algorithm is designed for general-purpose RLHF. In its $k=1$ configuration, it is designed for maximum efficiency and prompt diversity.

It optimizes the PPO objective in \Cref{eq:ppo} but redefines the advantage $A_{q,o_t}$. We use the standard PPO reward formulation, which incorporates the k1-style KL penalty directly into the reward:
\begin{equation}\label{eq:adv}
A_{q,o_t} = r(o_{1:T},q) - \beta \cdot \sum_{i=t}^{T} \mathrm{KL}(i)
\end{equation}
where $\text{KL}(t) = \log\left(\frac{\pi_{\theta_{\text{old}}}(o_t | q, o_{<t})}{\pi_{ref}(o_t | q, o_{<t})}\right)$.

The key innovation is our normalization strategy. Instead of a non-existent (for $k=1$) or biased local norm, we use Global Advantage Normalization~\citep{andrychowicz2020matters}:
\begin{equation}\label{eq:batch_norm}
A_{q,o_t}^{\text{norm}} = \frac{A_{q,o_t} - \text{mean}\left(A|A \in \mathcal{D}_{\text{batch}}\right)}{\text{std}\left(A|A \in \mathcal{D}_{\text{batch}}\right) + \epsilon}
\end{equation}
This global normalization is the core of \ours. As the global batch size $\mathcal{D}_{\text{batch}}$ is typically large (e.g., 1024 or more), the $\text{mean}(\cdot)$ and $\text{std}(\cdot)$ converge to stable constants. It makes the gradient estimator \textit{effectively less biased} (as $N \to \infty$) and robust to outliers, drastically improving training stability (see Appendix \Cref{appendix:whyglobal}). 
The algorithm for the efficient $k=1$ case is detailed in Algorithm \Cref{algo:main}.

\begin{algorithm}
\caption{\ours when $k=1$}\label{algo:main}
\begin{algorithmic}[1]
\Require Initial policy model $\pi_{ref}$, reward models $R$, task prompts $\mathcal{D}$
\State policy model $\pi_{\theta} \gets \pi_{ref}$
\For {step $= 1, \ldots, M$}
    \State Sample a batch $\mathcal{D}_{batch}$ from $\mathcal{D}$
    \State Update the old policy model $\pi_{\text{old}} \gets \pi_{\theta}$
    \State \textbf{Sample output ($k=1$)} $o \sim \pi_{\text{old}}(\cdot \mid q)$ for each question $q \in \mathcal{D}_{batch}$
    \State Compute rewards $r_i$ for each sampled $o_i$ using $R$
    \State Compute advantage $A_{q,o_t}$ via \Cref{eq:adv}
    \State Normalize advantages globally across $\mathcal{D}_{batch}$ to get $A_{q,o_t}^{\text{norm}}$ via \Cref{eq:batch_norm}
    \For {iteration $= 1, \ldots, k$}
        \State Update $\pi_{\theta}$ by maximizing the objective in \Cref{eq:ppo} using $A_{q,o_t}^{\text{norm}}$
    \EndFor
\EndFor
\Ensure $\pi_{\theta}$
\end{algorithmic}
\end{algorithm}

\subsection{Local Baseline}
For more complex tasks, such as multi-step reasoning, sampling multiple responses per prompt ($k > 1$) can be beneficial \citep{yue2025vapo}. For this, we introduce \varn. This variant combines the benefits of group-sampling with the stability of global normalization. The advantage calculation is a two-step process:
\begin{itemize}
    \item \textbf{Group Mean Subtraction for Reshaping}: We first subtract the group mean reward. It serves as a local baseline, reshaping rewards to be robust to different reward scales (e.g., $0/1$ vs. $-1/1$).
    \begin{equation}
     A'_{q,o_t} = R_{q,o_t} - \mathrm{mean}_{\text{group}}(R_{q,o_t})
    \end{equation}
    \item \textbf{Global Batch Normalization for Stability}: We then normalize this initial advantage using the \textbf{global batch statistics}, not the unstable local group statistics.
    \begin{equation}
     A_{q,o_t}^{\text{norm}} = \frac{A'_{q,o_t} - \mathrm{mean}_{\text{batch}}(A')}{\text{std}_{\text{batch}}(A') + \epsilon}
    \end{equation}
\end{itemize}
This combination avoids the bias and instability of local standard deviation in GRPO.

Furthermore, this variant employs a separate KL loss term for regularization. We adopt the $k_{2}$ estimator. As shown in Appendix \Cref{appendix:kl_design}, the $k_{2}$ estimator provides a stable,  unbiased gradient for the Reverse KL divergence, unlike the $k_{3}$ estimator (used in GRPO), which is an unstable approximation \citep{liu2025rethinking}. The final objective is:
\begin{equation}
\mathcal{L} = \mathcal{L}_{\text{PPO}}(A^{\text{norm}}) - \lambda \cdot \mathcal{J}_{k_{2}{\,\text{as loss}}}(\theta)
\end{equation}
where $\mathcal{J}_{k_{2}\,\text{as loss}}(\theta)=\mathbb{E}[\frac{1}{2}(\log\frac{\pi_{\theta}}{\pi_{ref}})^{2}]$.

\subsection{Relationship with PPO}
\varn can be viewed as a simplified and more stable variant of PPO. It is formally equivalent to a PPO agent where: (1) The critic network is removed; (2) The GAE parameters are set to $\lambda=1$ and $\gamma=1$; and (3) A two-step global batch normalization is used as the baseline instead of a learned value function.

\subsection{Summary}

To address these challenges, we propose \ours, a method centered on a simple yet powerful idea without the critic model: \textbf{Global Advantage Normalization}. Instead of normalizing within a prompt's local group, \ours normalizes the advantage function across the \textit{entire global training batch}. This approach is:
\begin{itemize}
    \item \textbf{Theoretical Stability}: As the global batch size grows (e.g., $N=1024$), the batch mean and standard deviation converge to constants, resulting in an \textit{effectively unbiased} (bias vanishes as $N \to \infty$) and low-variance advantage estimator (see Appendix \Cref{appendix:whyglobal}).
    \item \textbf{Training Efficiency}: It retains the critic-free architecture, significantly reducing computational and memory overhead compared to PPO.
    \item \textbf{Strong Generalization}: Using a global baseline prevents overfitting to specific prompts and encourages a more robust policy.
\end{itemize}

%% file: section/experiment.tex
\section{Experiments}\label{sec:experiment}

We evaluate our two algorithms, \ours and \varn, in their respective target domains using the OpenRLHF framework \citep{hu2024openrlhf}.

\subsection{General RLHF}
First, we evaluate the \ours in its $k=1$ configuration, focusing on general-domain RLHF, where efficiency and prompt diversity are key.

\paragraph{Experimental Setup}
We use an instruction-following policy model (Llama-3-8B-SFT) and refine it using a Bradley-Terry reward model \citep{bai2022training} trained on $\sim$700K human preference pairs. The policy is trained on 20,000 diverse prompts. We compare \ours ($k=1$) with other critic-free methods: GRPO ($k=4$), RLOO ($k=4$), and ReMax ($k=1$ + 1).

\begin{table*}[!htbp]
    \centering
    \begin{tabular}{l|cc|c}
    \toprule
                       & Score & Length  & Per Token    \\
    \midrule
    \textbf{\ours (k=1)} & 46.7        & 832   &  \textbf{0.0561} \\
    GRPO (k=4)         & \textbf{46.8} & 860       &  0.0544 \\
    RLOO (k=4)         & 44.6        & 866   &  0.0515 \\
    ReMax (k=1+1)      & 45.1        & \textbf{805}  &  0.0560 \\
    \bottomrule
    \end{tabular}
    \caption{Comparison on Chat-Arena-Hard \citep{li2024crowdsourced}. The single-sample \ours ($k=1$) achieves a top-tier score while being more token-efficient than the group-sampling GRPO.}
    \label{tab:general_exp}
\end{table*}

\paragraph{Experimental Results}
As shown in \Cref{tab:general_exp}, the single-sample \textbf{\ours ($k=1$)} achieves a score of 46.7, statistically tied with the group-sampling GRPO (46.8). However, it produces shorter responses (832 tokens vs. 860), resulting in a higher and more efficient per-token score (0.0561). This result highlights that for general tasks, group sampling ($k>1$) is not necessary and may even be suboptimal.

\paragraph{Results Analysis}
\Cref{fig:orm_reward_kl} shows the training dynamics. GRPO's reward rises quickly, but its KL divergence also increases rapidly, suggesting it is "hacking" the reward model. In contrast, \ours shows a more stable reward increase with a much smaller KL divergence. It highlights the stability of global normalization and its higher "KL-to-reward" conversion efficiency, avoiding the reward hacking seen in local normalization methods, which is often linked to length exploitation \citep{2024Length}.

\begin{figure}[!htbp]
    \centering
    \includegraphics[width=1.0\linewidth]{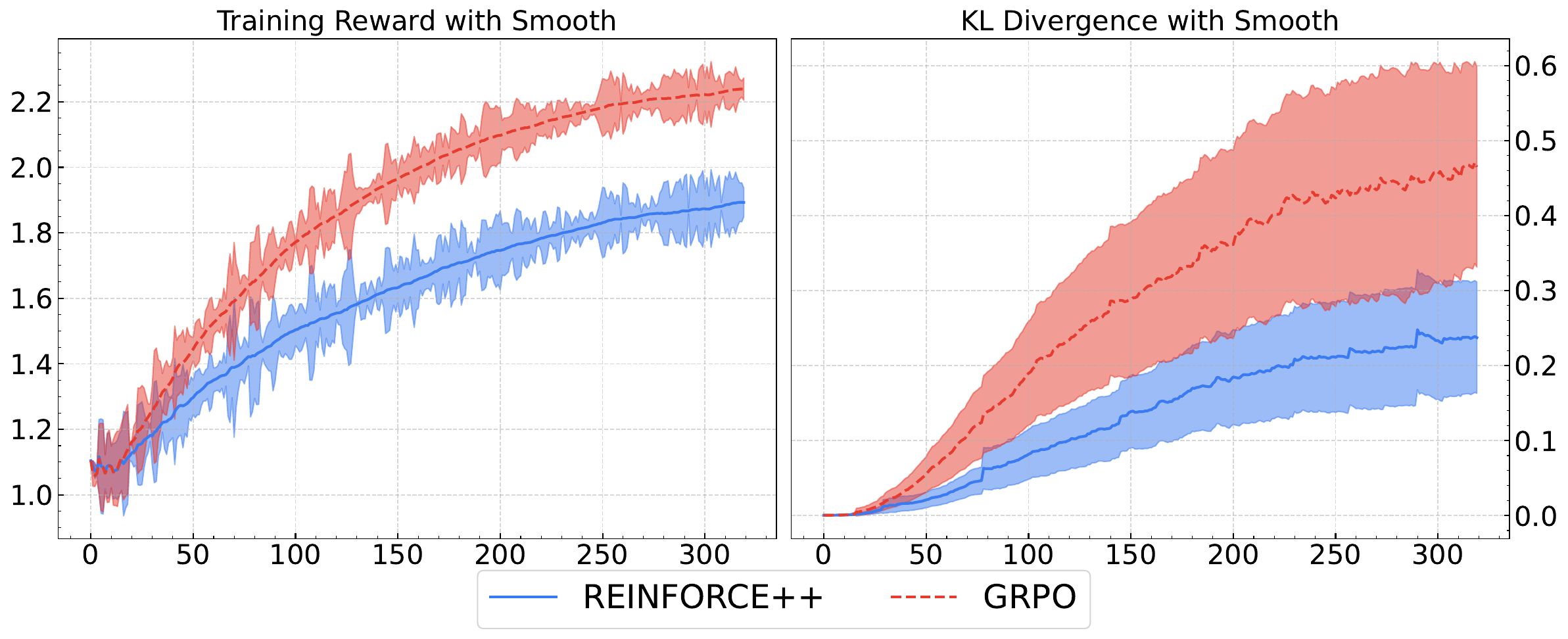}
    \caption{Comparison of smoothed Training Reward and KL Divergence. \ours ($k=1$) achieves strong rewards with significantly more stable (lower) KL divergence, avoiding GRPO's reward-hacking behavior.}
    \label{fig:orm_reward_kl}
\end{figure}



\begin{table*}[!ht]
    \centering
    \begin{tabular}{l|ccc}
    \toprule
                       & AIME-24 (Train) & \multicolumn{2}{c}{AIME-25 (Test)} \\
    Pass@N             & N = 1  &  N = 1     &  N = 16 \\
    \midrule
    GRPO               & \textbf{95.0}   & 0.0        & 0.4 \\
    \textbf{\ours ($k>1$)} & 71.0        & \textbf{2.5}& \textbf{40.0} \\ 
    \bottomrule
    \end{tabular}
    \caption{Comparison on a small training dataset. GRPO (local norm) overfits completely, while \ours (global norm) generalizes.}
    \label{tab:small_exp}
\end{table*}

\subsection{Reasoning Experiments} 
Next, we evaluate \ours (group-sample, $k>1$) in its target domain: complex reasoning, where group sampling is standard and reward signals are often sparse (e.g., 0/1). We compare it directly with GRPO ($k > 1$).

\subsubsection{Long Reasoning Task}
We use rule-based rewards (RLVR) in mathematical reasoning settings \citep{guo2025deepseek,seed2025seed1}.

\paragraph{Analysis on Small-Scale Datasets}
To test for overfitting, we trained on only 30 questions from AIME-24 and evaluated on AIME-25. As shown in \Cref{tab:small_exp}, GRPO (local norm) achieves a near-perfect 95.0\% on the training set but completely fails on the test set (0.0 Pass@1), demonstrating catastrophic overfitting. In contrast, \ours with global norm, while scoring lower on the training set (71.0), shows significantly better generalization (2.5 Pass@1, 40.0 Pass@16). 

\Cref{fig:details} confirms this. GRPO (left) immediately overfits, mastering the training questions in just a few steps. \ours (right) learns more gradually, enabled by the stable global normalization signal. 

\begin{figure*}[!htbp]
    \centering
    \begin{subfigure}[b]{0.5\linewidth}
        \centering
        \includegraphics[width=\textwidth]{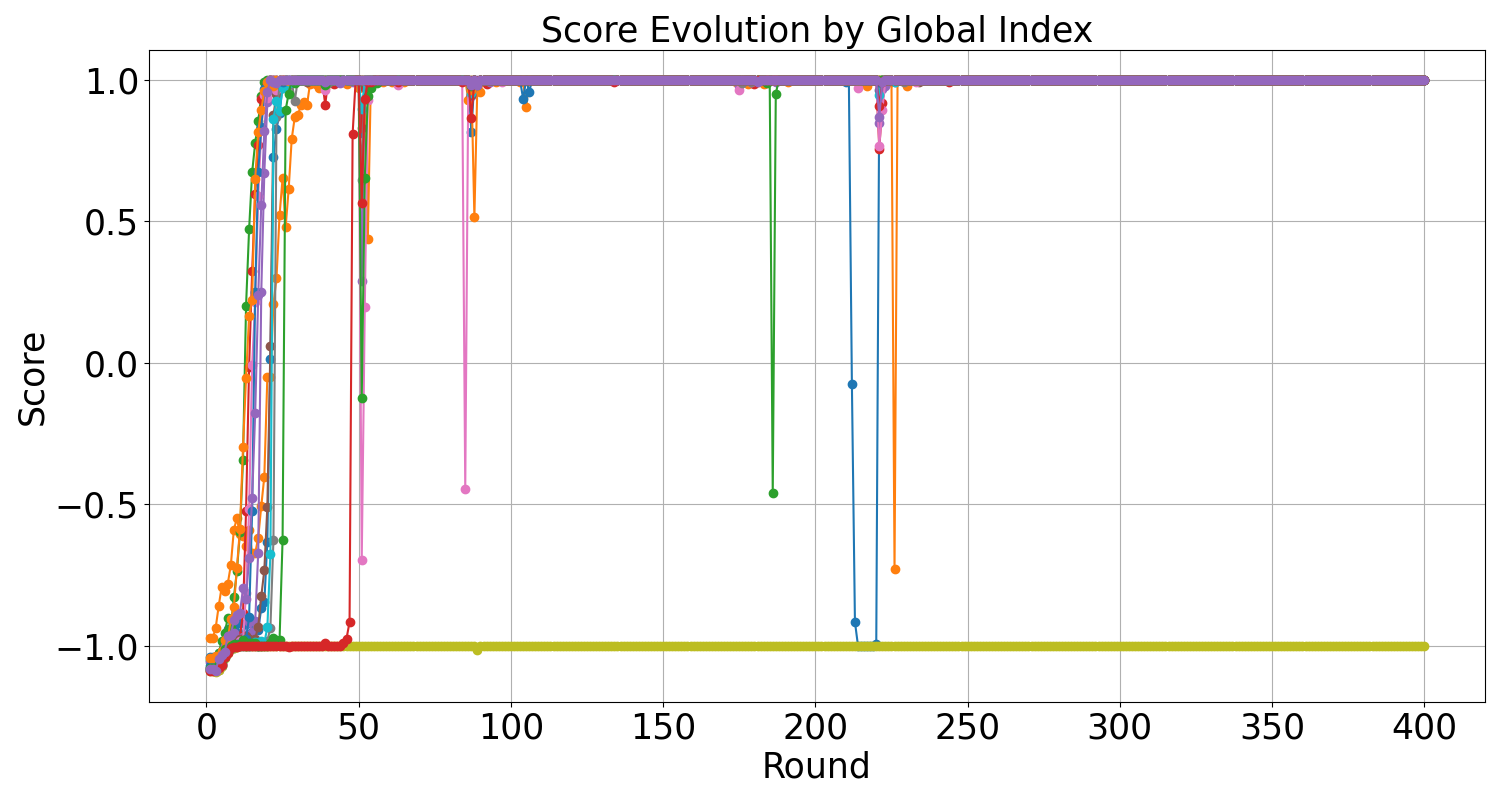}
    \end{subfigure}%
    \hfill
    \begin{subfigure}[b]{0.5\linewidth}
        \centering
        \includegraphics[width=\textwidth]{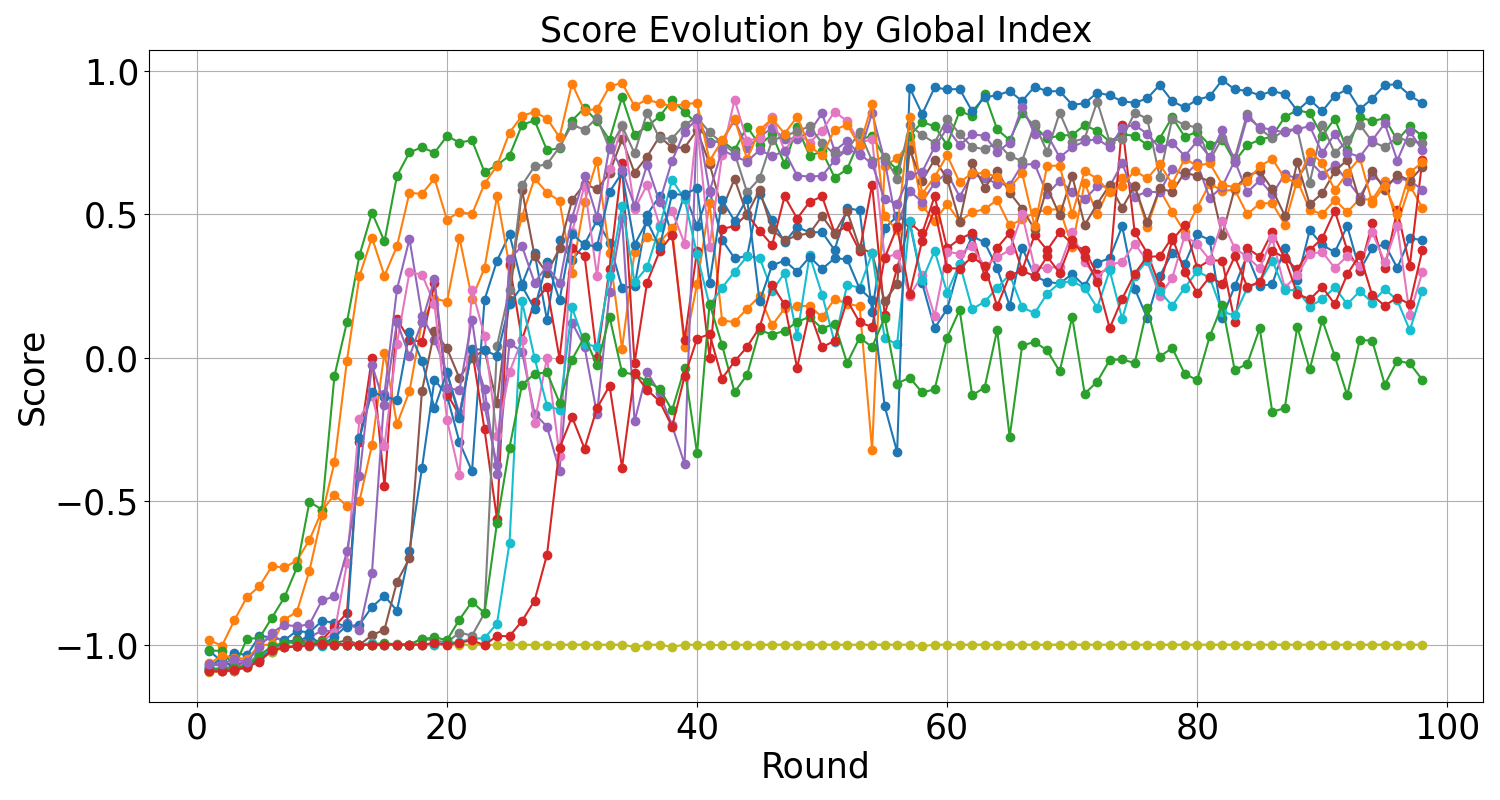}
    \end{subfigure}
    \caption{Training curves on a small prompt dataset. Left: GRPO (local norm) immediately overfits. Right: \ours (global norm) learns more stably.}
    \label{fig:details}
\end{figure*}

\paragraph{Logical Reasoning (K\&K Puzzles)}
We also tested on the Knights and Knaves (K\&K) puzzles \citep{xie2025logic, xie2024memorization}, where difficulty increases with the number of "people." \Cref{fig:logic_score} shows that while GRPO is competitive on easy tasks (2-3 people), its performance collapses on harder, OOD tasks (8 people). \ours is more robust, outperforming GRPO on all tasks with four or more people and achieving a much higher average score (62.1 vs. 55.7). 

\begin{figure}[!htbp]
    \centering
    \includegraphics[width=0.7\linewidth]{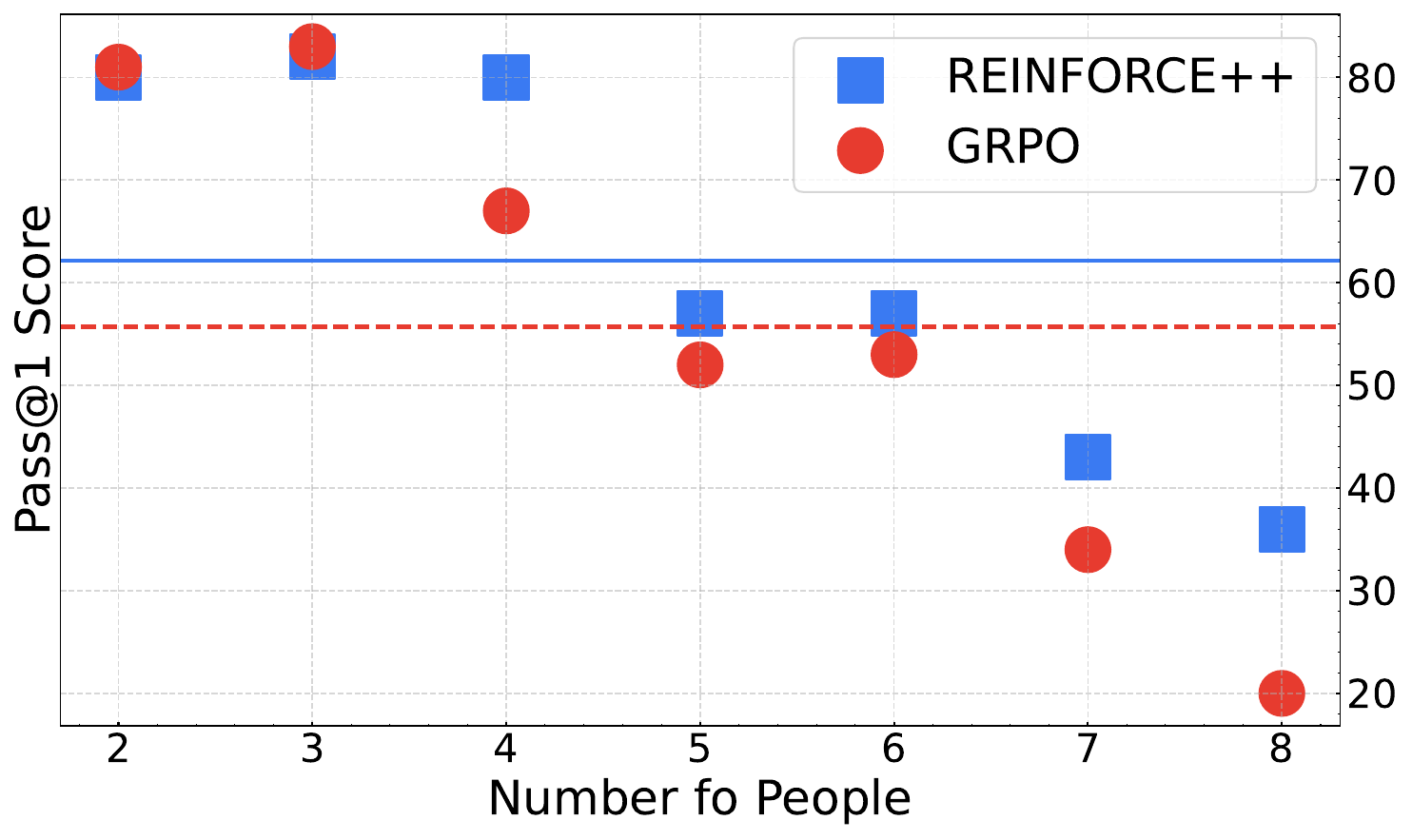}
    \caption{Comparison on logic benchmarks. \ours's global normalization provides better robustness to increasing task difficulty and OOD challenges (8 people).}
    \label{fig:logic_score}
\end{figure}

\paragraph{RL from Zero Setting}
Finally, we trained a Qwen2.5-Math-Base model from zero on MATH dataset splits \citep{guo2025deepseek}. \Cref{Tab:zero-experiment} shows that \ours again achieves better OOD generalization on the more challenging AIME-24 and AMC-23 datasets, while remaining competitive on the in-distribution MATH-500 test set. 

\begin{table*}[!htbp]
    \small
    \centering
    \begin{tabular}{l|ccc}
    \toprule
                       & AIME-24 (OOD) & AMC-23 (OOD) & MATH-500 (ID) \\
    Pass@N             & N = 8   &  N = 8   &  N = 1    \\
    \midrule
    GRPO               & 18.96       & 59.22    & \textbf{73.00}    \\
    \ours & \textbf{21.04}  & \textbf{60.47}   & 72.00     \\ 
    \bottomrule
    \end{tabular}
    \caption{Comparison on RL from Zero. \ours shows superior OOD performance.}
    \label{Tab:zero-experiment}
\end{table*}

\subsection{Multi-step Reinforcement Learning}
Finally, to test performance in a complex, multi-step environment, we utilize the zero-shot agent setup from \cite{mai2025agentrlscalinglaw}, where a Qwen 2.5 Base 7B model must learn to use Python tools to solve mathematical problems. It is a challenging scenario where group sampling ($k > 1$) and reward reshaping are crucial (see Appendix \Cref{appendix:best}).

\paragraph{Experimental Setup}
We adhere to the training and evaluation protocols of the ZeroTIR environment \cite{mai2025agentrlscalinglaw}. The backbone model is Qwen 2.5 Base 7B, trained using OpenRLHF \cite{hu2025openrlhfeasytousescalablehighperformance} with datasets from ORZ \cite{hu2025openreasonerzeroopensourceapproach} and DAPO \cite{yu2025dapoopensourcellmreinforcement}. We assess performance on benchmarks including AIME 2024 \cite{aime2024card}, AIME-2025 \cite{aime2025card}, HMMT FEB-2024/2025, and CMIMC, using the \texttt{average@32} metric.

\paragraph{Experimental Result}
As shown in \Cref{tab:performance_comparison_main}, \varn achieves the highest average accuracy (24.10) across all benchmarks. It significantly outperforms both GRPO (22.58) and the full-critic PPO (21.85). It demonstrates that our stable, critic-free approach is highly effective for complex agentic tasks, surpassing even the heavyweight PPO algorithm. The combination of group-mean reshaping, stable global normalization, and the correct `k2` KL loss proves to be a superior strategy.

\begin{table*}[!htbp]
\small
\centering
\begin{tabular}{l|ccccc|c}
\toprule
\textbf{Algorithm} & AIME 24 & AIME 25 & HMMT 2025 & HMMT 2024 & CMIMC & \textbf{Avg} \\
\midrule
GRPO (local norm) & \textbf{31.66} & 21.87 & 16.97 & 17.70 & 24.68 & 22.58 \\
PPO (critic-based) & 30.20 & 21.66 & 15.00 & 18.43 & 23.95 & 21.85 \\
\textbf{RF++-Baseline (global norm)} & 30.83 & \textbf{27.18} & \textbf{17.91} & \textbf{18.95} & \textbf{25.62} & \textbf{24.10} \\
\bottomrule
\end{tabular}
\caption{Performance Comparison on complex tool-use benchmarks (average@32). \varn outperforms both GRPO (local norm) and PPO (critic-based) models.}
\label{tab:performance_comparison_main}
\end{table*}

%% file: section/practise.tex
\section{Best Practices} \label{appendix:best}

\subsection{General Principle}

This paper introduces two algorithms: \ours and \varn[n], with a key focus on selecting the appropriate algorithm based on task conditions.

Empirical evidence from the open-source community suggests that \varn[n] is particularly effective for sample filtering or more complex scenarios. For instance, in the multi-turn tool-calling setting, a high proportion of void (i.e., non-informative or incorrect) samples can destabilize the training process. In such cases, incorporating a baseline by subtracting the intra-group mean reward significantly improves training stability by effectively filtering out void samples. Moreover, this automatic reward reshaping reduces the need for designing complex reward structures. For example, \varn[n] supports both 0/1 and -1/1 reward schemes, while \ours performs best with symmetric rewards, such as -1/1 in RLVR tasks.

In contrast, for general-domain tasks where prompt diversity and efficiency are paramount, or for tasks where obtaining multiple reward signals for distinct responses is challenging—such as training with Process-Supervised Reward Models~(PRMs) or online real-time sampling—we recommend using \textbf{\ours ($k=1$)}. Our experiments (Section 4.1) show that it provides superior OOD generalization in these settings.

\subsection{Third-Party Validation}
The core principle of \ours, global advantage normalization, has been independently validated and adopted in several \textbf{subsequent} large-scale reasoning systems, confirming its stability and effectiveness.

\paragraph{LitePPO}~\citep{liu2025part}, which effectively combines \varn with a token-level loss, conducted experiments demonstrating that the global standard deviation is superior to the local standard deviation used in GRPO. Their results show more stable training and better generalization, aligning perfectly with our findings.

\paragraph{ScaleRL}~\citep{khatri2025art} performed a detailed ablation study on advantage estimation methods in large-scale (16,000 GPU-hour) experiments. They directly compared batch-level normalization in \ours with prompt-level normalization in GRPO. Their findings concluded that batch-level normalization was "slightly superior in both compute efficiency and final performance," confirming the advantages of global normalization at scale.

\paragraph{DLER}~\citep{liu2025dler} found that when using truncation to control the output length of an LLM, batch-wise (global) normalization remains stable while group-wise (local) normalization shows declining accuracy. The study provides further evidence that global normalization is more robust to variations in training conditions and reward landscapes.

%% file: section/conclusion.tex
\section{Conclusion}\label{sec:conclusion}
In this paper, we present \ours, a critic-free RLHF framework designed to enhance training stability by addressing fundamental flaws in existing methods. We identified that prior critic-free algorithms, such as GRPO, rely on \textit{prompt-level (local)} advantage normalization, which we prove to be theoretically biased (Appendix \Cref{appendix:proof_GRPO}) and demonstrate to be practically unstable.

Our solution, \textbf{Global Advantage Normalization}, normalizes advantages across the entire batch, providing a stable and \textit{effectively unbiased} estimator (whose bias vanishes as batch size increases). We proposed two variants tailored for different needs:
\begin{enumerate}
    \item \ours: For general-domain RLHF, we show the single-sample ($k=1$) approach is highly efficient and achieves superior OOD generalization. We also show that its $k>1$ application is robust for reasoning tasks (Section 4.2).
    \item \varn: For complex reasoning and agent tasks, this variant combines group-mean reshaping with stable global normalization and a theoretically sound $k_2$ estimator for KL.
\end{enumerate}

Our empirical results, supported by independent validation from multiple third-party systems, confirmed this specialized approach. \ours demonstrated state-of-the-art generalization in general RLHF. \varn showed dramatic improvements in stability, preventing overfitting in low-data regimes and outperforming both GRPO and full-critic PPO in complex, long-horizon tool-use tasks. This work demonstrates that a theoretically sound and stable approach without the critic model can be more efficient and effective than traditional PPO.